\documentclass[twoside,11pt]{article}

\usepackage[preprint]{ewrl_2022}
\usepackage{macros}

\ewrlheading{15}{2022}{September 2022, Milan, Italy}{Author 1 and Author 2}

\ShortHeadings{Optimistic PAC Reinforcement Learning: the Instance-Dependent View}{Tirinzoni, Al-Marjani, and Kaufmann}
\firstpageno{1}

\begin{document}

\title{Optimistic PAC Reinforcement Learning: the Instance-Dependent View}

\author{\name Andrea Tirinzoni \email tirinzoni@fb.com \\
       \addr Meta AI\\
       Paris, France
       \AND
       \name Aymen Al-Marjani \email aymen.al\_marjani@ens-lyon.fr \\
       \addr UMPA, ENS Lyon\\
       Lyon, France
       \AND
       \name Emilie Kaufmann \email emilie.kaufmann@univ-lille.fr \\
       \addr Univ. Lille, Inria,  CNRS, Centrale Lille, UMR 9189 - CRIStAL\\
       Lille, France}

\maketitle

\begin{abstract}
Optimistic algorithms have been extensively studied for regret minimization in episodic tabular MDPs, both from a minimax and an instance-dependent view. However, for the PAC RL problem, where the goal is to identify a near-optimal policy with high probability, little is known about their instance-dependent sample complexity. A negative result of \cite{wagenmaker21IDPAC} suggests that optimistic sampling rules cannot be used to attain the (still elusive) \emph{optimal} instance-dependent sample complexity. On the positive side, we provide the first instance-dependent bound for an optimistic algorithm for PAC RL, BPI-UCRL,  for which only minimax guarantees were available \citep{Kaufmann21RFE}. While our bound features some minimal visitation probabilities, it also features a refined notion of sub-optimality gap compared to the value gaps that appear in prior work. Moreover, in MDPs with deterministic transitions, we show that  BPI-UCRL is actually near-optimal.  On the technical side, our analysis is very simple thanks to a new ``target trick'' of independent interest. We complement these findings with a novel hardness result explaining why the instance-dependent complexity of PAC RL cannot be easily related to that of regret minimization, unlike in the minimax regime.    
\end{abstract}



\section{Introduction}

We are interested in the probably approximately correct (PAC) identification of the best policy in an episodic Markov Decision Process (MDP) with finite state space $\cS$, action space $\cA$, and horizon $H$. We denote by  $\cM := (\cS, \cA, (p_h,\nu_h)_{h\in[H]}, s_1, H)$ such an MDP. Each episode starts in the initial state $s_1\in\cS$ and lasts $H$ steps (called stages). In each stage $h\in[H]$, the agent is in some state $s_h\in\cS$, it takes an action $a_h\in\cA$, it receives a random reward drawn from a distributions $\nu_h(s,a)$ with expectation $r_h(s,a)$, and it transitions to a next state $s_{h+1}\in\cS$ with probability $p_h(\cdot|s_h,a_h)$. A (deterministic) policy $\pi = (\pi_h)_{h \in [H]}$ is a sequence of mappings $\pi_{h} : \cS \rightarrow \cA$. The action-value function $Q_h^{\pi}(s,a)$ quantifies the expected cumulative reward when starting in state $s$ at stage $h$, taking action $a$ and following policy $\pi$ until the end of the episode. It satisfies the Bellman equations: for all $h \in [H]$, $s\in \cS$, and $a\in\cA$, 
\[Q_h^{\pi}(s,a) = r_h(s,a) + \sum_{s' \in \cS} p_h(s'|s,a) V^{\pi}_{h+1}(s'),\]
where $V_{h}^\pi(s) := Q_h^\pi(s,\pi_h(s))$ is the corresponding value function (with $V_{H+1}^{\pi} = 0$). A policy $\pi^\star$ is optimal if $V_1^{\pi^\star}(s_1) = \max_{\pi}V_1^{\pi}(s_1)$. From the theory of MDPs \citep{Puterman94MDP}, a sufficient condition is that $\pi^\star_h(s) \in \argmax_{a\in\cA}Q_h^\star(s,a)$, where the optimal Q-function satisfies $Q_h^\star(s,a) = r_h(s,a) + \sum_{s' \in \cS} p_h(s'|s,a) V^{\star}_{h+1}(s')$, with $V_h^\star(s) = \max_{a\in\cA}Q_h^\star(s,a)$ and $V_{H+1}^\star(s) = 0$. This condition implies that $\pi^\star$ maximizes the expected return at any state and stage simultaneously, while the (weaker) optimality condition only requires so at the initial state $s_1$.

In online episodic reinforcement learning (RL), the agent interacts with the MDP $\cM$ by choosing, in each episode $t\in\mathbb{N}$, a policy $\pi^{t}$ and collecting a trajectory in the MDP under this policy: $(s_h^{t},a_h^{t},r_h^{t})_{h\in[H]}$ where $s_1^{t} = s_1$ and, for all $h \in [H]$, $a_h^{t} = \pi_h^{t}(s_h^{t})$, $r_h^t \sim \nu_h(s_h^t,a_h^t)$, and $s_{h+1}^{t} \sim p_h(\cdot | s_h^t,a_h^{t})$. The choice of $\pi^{t}$ based on previously observed trajectories is called the \emph{sampling rule}. Several objectives have been studied in the literature. An agent seeking to maximize the total reward received in $T$ episodes equivalently aims at minimizing the (pseudo) regret 
\[\mathcal{R}_{\cM}(T) := \sum_{t=1}^{T} \left(V_1^\star(s_1)  - V_1^{\pi^{t}}(s_1)\right).\]
In PAC identification (or PAC RL), the agent's sampling rule is coupled with a (possibly adaptive) stopping rule $\tau$ after which the agent stops collecting trajectories and returns a guess for the optimal policy $\widehat{\pi}$. Given two parameters $\varepsilon,\delta >0$ with $\delta \in (0,1)$, the algorithm $((\pi^t)_{t\in\mathbb{N}}, \tau, \widehat{\pi})$ is $(\varepsilon,\delta)$-PAC if it returns an $\varepsilon$-optimal policy with high probability, i.e.,
\[\mathbb{P}_{\cM}\left(V_1^{\widehat{\pi}}(s_1) \geq V_1^\star(s_1) - \epsilon\right) \geq 1 - \delta.\]
The goal is to have $(\varepsilon,\delta)$-PAC algorithms using a small number of exploration episodes $\tau$ (a.k.a. sample complexity). 

The PAC RL framework was originally introduced by \cite{Fiechter94} and there exists algorithms attaining a sample complexity $O((SAH^3/\varepsilon^2)\log(1/\delta))$ \citep{dann15PAC,Menard21RFE}, which is optimal in a minimax sense in time-inhomogeneous MDPs \citep{Omar21LB}. These algorithms use an \emph{optimistic} sampling rule coupled with a well-chosen stopping rule. Optimistic sampling rules, in which the policy $\pi^t$ is the greedy policy with respect to an upper confidence bound on the optimal Q function, have been mostly proposed for regret minimization (see \cite{Neu20Optimism} for a survey). In particular, the UCBVI algorithm of \cite{Azar17UCBVI} (with Bernstein bonuses) attains minimax optimal regret in episodic MDPs. Recent works have provided instance-dependent upper bounds on the regret for optimistic algorithms \citep{simchowitz2019non,xu2021fine,dann21ReturnGap}. An instance-dependent bound features some complexity term which depends on the MDP instance, typically through some notion of sub-optimality gap. To the best of our knowledge, for PAC RL in episodic MDPs the only algorithms with instance-dependent upper bound on their sample complexity are MOCA \citep{wagenmaker21IDPAC} and EPRL \citep{tirinzoni22deterministic}, the latter being analyzed for MDPs with deterministic transitions. Neither of these algorithms are based on an optimistic sampling rule.

Notably, \cite{wagenmaker21IDPAC} proved that no-regret sampling rules (including optimistic ones) cannot achieve the instance-optimal rate for PAC identification. The intuition is quite simple: an optimal algorithm for PAC RL must visit every state-action pair at least a certain amount of times, and this requires playing policies that cover the whole MDP in the minimum amount of episodes. On the other hand, a regret-minimizer focuses on playing high-reward policies which, depending on the MDP instance, might be arbitrarily bad at visiting hard-to-reach states. 

Despite not being instance-optimal, optimistic sampling rules are simple (e.g., as opposed to the complex design of MOCA), computationally efficient, and do not require any sophisticated elimination rule (e.g., as opposed to the one proposed by \cite{tirinzoni22deterministic} to obtain the optimal gap dependence in deterministic MDPs). However, it remains an open question what instance-dependent complexity they can achieve.

\paragraph{Contributions}

Our main contribution is a new instance-dependent analysis for (a variant of) BPI-UCRL, a PAC RL algorithm based on an optimistic sampling rule proposed by \cite{Kaufmann21RFE} with only a worst-case sample complexity bound. In particular, in Theorem \ref{th:main-sample-complexity} we show that the sample complexity of BPI-UCRL can be bounded by
\begin{align*}
    \tau \lesssim \sum_{h\in[H]}\sum_{s\in\cS}\sum_{a\in\cA} \frac{H^4\log(1/\delta)}{p_h^{\min}(s,a)\max\{\widetilde{\Delta}_h(s,a),\epsilon\}^2},
\end{align*}
where $p_h^{\min}(s,a)$ is the minimum positive probability to reach $(s,a)$ at stage $h$ across all deterministic policies, while $\widetilde{\Delta}_h(s,a) := \min_{\pi:p_h^\pi(s,a) > 0}\max_{\ell \in [H]}\max_{s' : p_{\ell}^{\pi}(s')>0} ( V^\star_\ell(s') - V_\ell^\pi(s'))$ is a new notion of sub-optimality gap that we call the \emph{conditional return gap}.\footnote{We denote by $p_h^\pi(s,a)$ (resp. $p_h^\pi(s)$) the probability that $\pi$ visits $(s,a)$ (resp. $s$) at stage $h$.} Interestingly, we show that the gaps $\widetilde{\Delta}_h(s,a)$ are larger than both the value gaps of \cite{wagenmaker21IDPAC} and of the (deterministic) return gaps of \cite{tirinzoni22deterministic}. Notably, we prove this result with a remarkably simple analysis based on a new ``target trick'': instead of bounding the number of times each state-action-stage triplet $(s,a,h)$  is visited (as it is common in the bandit literature), we control the number of times the played policy visits $(s,a,h)$ with positive probability with $(s,a,h)$ being the least visited triplet so far, an event that we refer to as $(s,a,h)$ being ``targeted''. 

Our second contribution is to prove that, unlike what happens in the minimax setting, there is no clear relationship between regret and sample complexity in the instance-dependent framework. Indeed, the ``regret-to-PAC conversion'' often proposed to turn a regret minimizer into an $(\varepsilon,\delta)$-PAC algorithm for PAC RL \citep[e.g.,][]{Jin18OptQL,Menard21RFE,wagenmaker21IDPAC} cannot directly exploit an instance-dependent upper bound on the regret. In Theorem~\ref{th:regret-vs-bpi}, we construct an MDP for which the sample complexity suggested by a regret-to-PAC conversion cannot be attained by any $(\varepsilon,\delta)$-correct algorithm for PAC RL. In particular, this implies that one cannot take an instance-dependent regret bound for an optimistic algorithm \citep[e.g.,][]{simchowitz2019non} and turn it into an instance-dependent sample complexity bound of the form above: a specific analysis for PAC RL, like the one proposed in this paper, is actually required.


\section{The BPI-UCRL Algorithm}\label{sec:bpi-ucrl}

Let $n_h^t(s,a) := \sum_{j=1}^t \indi{s_h^j=s,a_h^j=a}$ be the number of times the state-action pair $(s,a)$ has been visited at stage $h$ up to episode $t$. We introduce the maximum-likelihood estimators \[\widehat{r}_h^t(s,a) := \frac{1}{n_h^t(s,a)}\sum_{j=1}^t \indi{s_h^j=s,a_h^j=a}r_h^j \ \ \text{ and } \ \ \widehat{p}_h^t(s'|s,a) := \frac{1}{n_h^t(s,a)}\sum_{j=1}^t \indi{s_h^j=s,a_h^j=a,s_{h+1}^j=s'}\] for $r_h(s,a)$ and $p_h(s'|s,a)$, respectively. As common, and without loss of generality, we shall assume that reward distributions are supported on $[0,1]$. We define inductively the following upper and lower bounds on the optimal value function. Letting $\overline{Q}^{t}_{H+1} = \underline{Q}^{t}_{H+1} = 0$, for all $h \in [H]$ we have
\begin{align*}
    \overline{Q}_h^{t}(s,a) &= \min\bigg(H-h+1,\ \widehat{r}_h^t(s,a) + b_h^t(s,a) + \sum_{s' \in \cS}\widehat{p}_h^t(s' |s,a) \overline{V}_{h+1}^{t}(s')\bigg), \quad &&\overline{V}_h^{t}(s) = \max_{a\in\cA}\overline{Q}_h^{t}(s,a),\\
    \underline{Q}_h^{t}(s,a) &= \max\bigg(0,\ \widehat{r}_h^t(s,a) - b_h^t(s,a) + \sum_{s' \in \cS}\widehat{p}_h^t(s' |s,a) \underline{V}_{h+1}^{t}(s')\bigg),\quad &&\underline{V}_h^{t}(s) = \max_{a\in\cA}\underline{Q}_h^{t}(s,a),
\end{align*}
where $b_h^t(s,a)$ is a confidence bonus defined as
\begin{align*}
    b_h^t(s,a) := (H-h+1) \left( \sqrt{\frac{{\beta}(n_h^{t}(s,a),\delta)}{n_h^{t}(s,a)}} \wedge 1 \right)
\end{align*}
for a suitable threshold ${\beta}$ that we shall specify in the analysis. The BPI-UCRL algorithm \citep{Kaufmann21RFE} can be described as follows:\footnote{The original BPI-UCRL algorithm uses slightly different Q-function bounds which do not feature $\widehat{r}_h^t(s,a)$ and $\widehat{p}_h^t(s'|s,a)$ explicitly but rather scale with KL confidence regions around them (see Appendix D of \cite{Kaufmann21RFE}). Here we write the explicit version obtained by appyling Pinsker's inequality, though our analysis also holds for the original confidence intervals.}
\begin{itemize}
    \item the \textbf{sampling rule} prescribes $\pi^{t+1}_h(s) = \argmax_{a\in\cA} \overline{Q}_h^{t}(s,a)$ for each $t\in\mathbb{N}$;
    \item the \textbf{stopping rule} is $\tau = \inf \left\{ t \in \N : \max_{a} \overline Q_1^{t}(s_1,a) - \max_{a} \underline Q_1^{t}(s_1,a)\leq \varepsilon\right\}$;
    \item the \textbf{recommendation rule} is $\widehat{\pi}^\tau_h(s) = \argmax_{a\in\cA} \underline{Q}_h^\tau(s,a)$.
\end{itemize}
Note that the sampling rule of BPI-UCRL is essentially the UCBVI algorithm with Hoeffding's bonuses proposed by \cite{azar2017minimax} for regret minimization. Such bonuses can be improved using Bernstein's inequality, yielding either UCBVI with Bernstein's bonuses \citep{azar2017minimax} or EULER \citep{Zanette19Euler}. While this would likely reduce the dependence on the horizon from $H^4$ to $H^3$ in our final sample complexity bound, we focus on Hoeffding's bonuses for simplicity since the extension to Bernstein's bonuses is somewhat straightforward given existing analyses.


\section{An Instance-dependent Analysis of BPI-UCRL}\label{sec:analysis}

Before stating and proving our main result, we introduce our novel notion of sub-optimality gap. Formally, the \emph{conditional return gap} of any state-action pair $(s,a)$ at stage $h\in[H]$ is
\begin{align}\label{eq:gaps}
    \widetilde{\Delta}_h(s,a) := \min_{\pi\in\Pi:p_h^\pi(s,a) > 0}\max_{\ell\in[H]}\max_{s'\in\cS : p_\ell^\pi(s') > 0} \big( V^\star_\ell(s') - V_\ell^\pi(s') \big),
\end{align}
where we recall that $p_h^\pi(s) := \mathbb{P}^\pi(s_h=s)$ and $p_h^\pi(s,a) = p_h^\pi(s)\indi{\pi_h(s)=a}$. The intuition behind this definition is quite simple: in order to figure out whether $(s,a)$ is sub-optimal at stage $h$, the agent must learn that all policies visiting $(s,a)$ at stage $h$ with positive probability are indeed sub-optimal. The complexity for detecting whether any of such policies (say, $\pi$) is sub-optimal is proportional to the maximum gap between the optimal value function and the one of $\pi$ across all possible states visited by $\pi$ itself. This is a gap between expected returns conditioned on different starting states and stages (hence the name conditional return gap). It turns out that these gaps are larger than both the value gaps $\Delta_h(s,a) := V^\star_h(s) - Q^\star_h(s,a)$ \citep{wagenmaker21IDPAC} and the variant of the return gaps $\overline{\Delta}_h(s,a) = V_1^\star(s_1) - \max_{\pi\in\Pi:p_h^\pi(s,a) > 0}V_1^\pi(s_1)$ introduced by \cite{tirinzoni22deterministic}.\footnote{The return gaps were introduced by \cite{tirinzoni22deterministic} only for deterministic MDPs. Here we replace their maximum over policies visiting $(s,a,h)$ with probability 1 by the one over policies visiting it with positive probability.}
\begin{proposition}\label{prop:gaps-comparison}
    For all $s\in\cS,a\in\cA,h\in[H]$, $\widetilde{\Delta}_h(s,a) \geq \Delta_h(s,a)$ and $\widetilde{\Delta}_h(s,a) \geq \overline{\Delta}_h(s,a)$. Moreover, if the MDP has deterministic transitions, $\widetilde{\Delta}_h(s,a) = \overline{\Delta}_h(s,a)$.
\end{proposition}
\begin{proof}
    For the first inequality, we have
    \begin{align*}
        \widetilde{\Delta}_{h}(s,a) \geq V_h^\star(s) - \max_{\pi: p_{h}^{\pi}(s,a) > 0} V_h^{\pi}(s) = V_h^\star(s) - \max_{\pi: p_{h}^{\pi}(s,a) > 0} Q_h^{\pi}(s,a)
        = V_h^\star(s) - Q_h^{\star}(s,a) = \Delta_h(s,a).
    \end{align*}
    The second one is trivial by lower bounding the maximum with $s'=s_1$ and $\ell=1$. To see the equality, note that $V^\star_h(s) - V_h^\pi(s) = \bE^\pi\left[ \sum_{\ell=h}^H \Delta_\ell(s_\ell,\pi_\ell(s_\ell)) \mid s_h=s \right]$. In the deterministic case, this implies that $V^\star_h(s) - V_h^\pi(s)$ is a sum of $H-h+1$ fixed (non-negative) value gaps. Therefore, the maximum in \eqref{eq:gaps} must be attained at the initial stage and state, which implies the statement.
\end{proof}

The first return gaps were actually introduced in the regret-minimization literature by \cite{dann21ReturnGap} as
\[\overline{\text{gap}}_h(s,a) = \Delta_h(s,a) \vee \frac{1}{H}\min_{\pi\in\Pi : \bP(\cB_h(s,a))>0} \bE^{\pi}\left[\sum_{\ell = 1}^{h} \Delta_{\ell}(s_{\ell},a_{\ell}) \mid \cB_h(s,a)\right],\]
where $\cB_h(s,a) = \{s_h = s, a_h = a, \exists \ell \leq h : \Delta_{\ell}(s_{\ell},a_{\ell}) > 0\}$ is the event that policy $(s,a)$ is visited at stage $h$ after at least one mistake was made. We found no clear relationship between $\overline{\text{gap}}_h(s,a)$ and $\widetilde{\Delta}_h(s,a)$ besides the fact that the former is also comparing returns (from stage 1), as for any policy playing optimally from stage $h+1$, $V^{\star}(s_1) - V^{\pi}(s_1) = \bE^{\pi}\left[\sum_{\ell=1}^{h}\Delta_{\ell}(s_{\ell},a_{\ell})\right]$. We now state and prove our main result.

\begin{theorem}\label{th:main-sample-complexity}
    Let $\beta(t,\delta) := (\sqrt{\beta^r(t,\delta)} + \sqrt{2\beta^p(t,\delta)})^2$, where $\beta^{r}(t,\delta) := \frac{1}{2}(\log(3SAH/\delta) + \log(e(1+t)))$ and $\beta^{p}(t,\delta)  :=  \log(3SAH/\delta) + (S-1)\log(e(1+t/(S-1)))$. With probability at least $1-\delta$, BPI-UCRL outputs a policy $\widehat{\pi}^{\tau}$ satisfying $V_1^{\widehat\pi^{\tau}}(s_1)\geq V_1^{\star}(s_1) - \varepsilon$ using a number of episodes upper bounded as
    \begin{align*}
        \tau \leq H^4\sum_{h=1}^H \sum_{s\in\cS}\sum_{a\in\cA} \frac{720\log\frac{3SAH}{\delta} + 1729S\log\left(1152\frac{S^2AH^5}{p_{\min}\epsilon^2}\log\frac{3SAH}{\delta}\right)}{p_{h}^{\min}(s,a) \max \{\widetilde{\Delta}_{h}(s,a),\epsilon \}^2},
    \end{align*}
    where $p_{h}^{\min}(s,a) := \min_{\pi\in\Pi: p_{h}^{\pi}(s,a) > 0} p_{h}^{\pi}(s,a)$ and $p_{h}^{\min}(s,a) = +\infty$ when $(s,a,h)$ is unreachable by any policy.
\end{theorem}

Theorem \ref{th:main-sample-complexity} shows that the sample complexity of BPI-UCRL is upper bounded by a function that scales inversely with the conditional return gaps squared multipled by the minimum visitation probabilities of each triplet $(s,a,h)$. We recall that BPI-UCRL also enjoys the worst-case sample complexity bound $\tau \leq \widetilde{O}(SAH^4\log(1/\delta)/\epsilon^2)$ proved by \cite{Kaufmann21RFE}, which is minimax optimal up to a factor $H$. Thus, one can always take the minimum between this worst-case bound and the instance-dependent one of Theorem \ref{th:main-sample-complexity}. Before proving our main theorem, we briefly discuss how it relates to existing results.

\paragraph{Comparison to \cite{wagenmaker21IDPAC}} The sample complexity upper bound achieved by the MOCA algorithm of \cite{wagenmaker21IDPAC} is roughly
\begin{align*}
    \tau \leq \widetilde{O}\left(H^2\log(1/\delta) \sum_{h\in[H]} \sum_{s\in\cS}\sum_{a\in\cA} \min\left(\frac{1}{p_{h}^{\max}(s,a){\Delta}_{h}(s,a)^2}, \frac{p_{h}^{\max}(s,a)}{\epsilon^2}\right) + \frac{H^4 |\mathrm{OPT}(\epsilon)|\log(1/\delta)}{\epsilon^2}\right),
\end{align*}
where $p_{h}^{\max}(s,a) := \max_{\pi\in\Pi: p_{h}^{\pi}(s,a) > 0} p_{h}^{\pi}(s,a)$ and $\mathrm{OPT}(\epsilon)$ is roughly the set of all $\epsilon$-optimal triplets. In contrast to the bound we obtained for BPI-UCRL, this one scales with the maximum probabilities for reaching the different state-action pairs. This is obtained thanks to the clever exploration strategy of MOCA which focuses on efficiently covering the whole MDP. However, the bound of \cite{wagenmaker21IDPAC} scales with value gaps which, from Proposition \ref{prop:gaps-comparison}, are provably smaller than our conditional return gaps. Overall, the two bounds result non-comparable as there exist MDP instances where the one of BPI-UCRL is smaller, and viceversa for the one of MOCA. While we are able to show this improved gap dependence thanks to optimism alone, we are not sure how to achieve it with a suitable elimination rule that could be plugged into the MOCA exploration strategy to obtain the best of these two bounds.

\paragraph{The dependence on $p_{h}^{\min}(s,a)$}

One might be wondering whether a better dependence than $p_{h}^{\min}(s,a)$ can be achieved with an optimistic rule like BPI-UCRL. We conjecture that this is not possible, at least in a worst-case sense. In fact, \cite{wagenmaker21IDPAC} already proved that there exists an MDP instance in which any no-regret sampling rule (thus including optimistic ones) suffers a depence on the minimum visitation probabilities, while a ``smart'' PAC RL algorithm does not. The intuition is that a no-regret algorithm focuses on playing high-reward policies which, depending on the MDP instance, might be arbitrarily bad at exploring the state space. In our context, this means that, if the policy visiting $(s,a,h)$ with largest reward is also the one that visits it with lowest probability, an optimistic sampling rule is likely to play such policy quite frequently and thus its sample complexity will scale inversely by $p_{h}^{\min}(s,a)$ as we show.

\paragraph{Deterministic MDPs (comparison to \cite{tirinzoni22deterministic})}

If the MDP has deterministic transitions, we have $\widetilde{\Delta}_h(s,a) = \overline{\Delta}_h(s,a)$ (see Proposition \ref{prop:gaps-comparison}) and $p_h^{\min}(s,a) = 1$ if state $s$ is reachable by some policy at stage $h$, while $p_h^{\min}(s,a) = +\infty$ in the opposite case. Theorem \ref{th:main-sample-complexity} then implies that 
\begin{align*}
    \tau \leq \widetilde{O}\left(H^4\sum_{h\in[H]} \sum_{s\in\cS_h}\sum_{a\in\cA} \frac{\log(1/\delta) + S\log\log(1/\delta)}{\max \{\overline{\Delta}_{h}(s,a),\epsilon \}^2}\right),
\end{align*}
where $\cS_h$ is the subset of states reachable at stage $h$. Up to the extra multiplicative $H^2$ and $S\log\log(1/\delta)$ terms, this matches the bound obtained by \cite{tirinzoni22deterministic} for the EPRL algorithm with a maximum-diameter sampling rule that is informed a-priori about the MDP being deterministic. These extra terms arise because BPI-UCRL needs to concentrate the transition probabilities to work for general stochastic MDPs. If we knew that the MDP is deterministic, we could modify the bonuses as $b_h^t(s,a) := \sqrt{\frac{{\beta}(n_h^{t}(s,a),\delta)}{n_h^{t}(s,a)}} \wedge 1$ and the thresholds as $\beta(t,\delta) := \beta^r(t,\delta)$. This would yield sample complexity $\tau \leq \widetilde{O}\left(H^2\sum_{h\in[H]} \sum_{s\in\cS_h}\sum_{a\in\cA} \log(1/\delta)/\max \{\overline{\Delta}_{h}(s,a),\epsilon \}^2\right)$ which matches exactly the one of EPRL with maximum-diameter sampling and which is at most a factor of $H^3$ sub-optimal w.r.t. the instance-dependent lower bound of \cite{tirinzoni22deterministic}. This is quite remarkable since EPRL obtains the ``optimal'' dependence on the deterministic return gaps $\overline{\Delta}_{h}(s,a)$ using a clever elimination rule, while here we show optimism alone is enough. We note, however, that reducing the sub-optimal dependence on $H^3$ still requires smarter exploration strategies than optimism, like the maximum-coverage one proposed by \cite{tirinzoni22deterministic}.

\subsection{Proof of Theorem \ref{th:main-sample-complexity}}

All lemmas and proofs not explicitly reported here can be found in Appendix \ref{app:analysis}.

We carry out the proof under the ``good event'' $\cE := \cE^r \cap \cE^p \cap \cE^c$, where
\begin{align*}
        \cE^r &:= \left\{\forall t\in\mathbb{N}_{>0},s\in\cS,a\in\cA,h\in[H] : \Big| r_h(s,a) - \widehat{r}^{t}_h(s,a) \Big| \leq \sqrt{\frac{\beta^{r}(n_h^t(s,a),\delta)}{n_h^{t}(s,a)\vee 1}}\right\},
        \\ \cE^p &:= \left\{\forall t\in\mathbb{N}_{>0},s\in\cS,a\in\cA,h\in[H] : \mathrm{KL}\left(\widehat{p}^{t}_h(\cdot | s,a) , p_h(\cdot | s,a)\right) \leq {\frac{\beta^{p}(n_h^t(s,a),\delta)}{n_h^{t}(s,a)\vee 1}}\right\},
        \\ \cE^c &:= \left\{\forall t\in\mathbb{N}_{>0},s\in\cS,a\in\cA,h\in[H] : n_h^t(s,a) \geq \frac{1}{2}\overline{n}_h^t(s,a) - \log(3SAH/\delta)\right\}.
\end{align*}
    Note that event $\cE^c$ relates the counts $n_h^{t}(s,a)$ to the pseudo-counts $\overline{n}_{h}^{t}(s,a) := \sum_{j=1}^{t} p_{h}^{\pi^j}(s,a)$.
    Thanks to Lemma \ref{lem:good-event}, we have $\bP(\cE) \geq 1 - \delta$ and, thus, the final result will hold with the same probability. 
    
    This good event is  identical to the one used in the original (minimax) analysis of BPI-UCRL \citep{Kaufmann21RFE}. On this good event, one can prove that our (slighlty different) bounds $\overline{Q}_h^{t}(s,a), \underline{Q}_h^{t}(s,a)$ are respectively upper and lower bounds on the optimal action value $Q^\star_h(s,a)$, for all $(s,a,h)$ (see Lemma~\ref{lem:conf-Q}, which justifies the choice of threshold $\beta$). The correctness follows from this fact using the same arguments as Theorem 11 of \cite{Kaufmann21RFE}. The original part of our proof is the way we upper bound the sample complexity on the good event $\cE$.   
    
Our proof is based on the following ``target trick'' which extends the one used by \cite{tirinzoni22deterministic} to MDPs with stochastic transitions. Fix any $s\in\cS$, $a\in\cA$, and $h\in[H]$. Let us introduce the event ``$(s,a,h)$ is targeted at time $t$'' as
\begin{align*}
    G^t_{s,a,h} := \left\{ p_{h}^{\pi^t}(s,a) > 0, (s,a,h) \in \argmax_{(s',a',\ell) : p_{\ell}^{\pi^t}(s',a') > 0} b_\ell^{t-1}(s',a') \right\}.
\end{align*}
Intutively, $(s,a,h)$ is targeted at time $t$ if (1) it is visited with positive probability by $\pi^t$ and (2) it maximizes the bonuses at time $t-1$ (i.e., the current uncertainty) across all triplets visited by $\pi^t$. Let $Z_{h}^\tau(s,a) := \sum_{t=1}^\tau \indi{G_{s,a,h}^t}$ be the number of times $(s,a,h)$ is targeted up the stopping time. Since at each time step at least one triplet is targeted,
\begin{align}\label{eq:tau-target-trick}
\tau \leq \sum_{h=1}^H \sum_{s\in\cS}\sum_{a\in\cA} Z_h^{\tau-1}(s,a) + 1.
\end{align}
We shall now focus on bounding $Z_h^T(s,a)$ for some time $T>0$ at the end of which the algorithm did not stop. Thanks to \eqref{eq:tau-target-trick}, this will imply a bound on the final stopping time.

We first state the following crucial result which relates confidence intervals and conditional return gaps.

\begin{lemma}\label{lem:policy-value-gap-vs-conf}
    Under event $\cE$, for any $t\in\mathbb{N}_{>0},s\in\cS,h\in[H]$,
    \begin{align*}
        V_h^\star(s) - V_h^{\pi^{t+1}}(s) \leq 2\sum_{\ell=h}^H\sum_{s'\in\cS} p_\ell^{\pi^{t+1}}(s'| s,h) b_\ell^t(s',\pi_\ell^{t+1}(s')),
    \end{align*}
    where $p_\ell^{\pi}(s'| s,h) := \bP^{\pi}(s_\ell=s' | s_h=s)$.
\end{lemma}
Let $(\tilde{s}_t,\tilde{h}_t) \in \argmax_{(s',\ell):p_\ell^{\pi^t}(s')>0} (V_\ell^\star(s') - V_\ell^{\pi^{t}}(s'))$. Using Lemma \ref{lem:policy-value-gap-vs-conf} with this couple,
\begin{align*}
    \max_{\ell\in[H]}\max_{s'\in\cS : p_\ell^{\pi^t}(s')>0}\big(V_\ell^\star(s') - V_\ell^{\pi^{t}}(s')\big) \leq 2\sum_{\ell=\tilde{h}_t}^H\sum_{s'\in\cS} p_\ell^{\pi^{t}}(s'| \tilde{s}_t,\tilde{h}_t) b_\ell^{t-1}(s',\pi_\ell^{t}(s')).
\end{align*}
Summing both sides for all episodes where $(s,a,h)$ is targeted up to $T$ and using that $p_h^{\pi^t}(s,a) > 0$ under $G_{s,a,h}^t$,
\begin{align}\label{eq:ineq-target-trick}
    2\sum_{t=1}^T \indi{G_{s,a,h}^t} \sum_{\ell=\tilde{h}_t}^H\sum_{s'\in\cS} p_\ell^{\pi^{t}}(s'| \tilde{s}_t,\tilde{h}_t) b_\ell^{t-1}(s',\pi_\ell^{t}(s'))
     \geq Z_{h}^T(s,a) \widetilde{\Delta}_h(s,a).
\end{align}
Note that, for each time $t$, since $p_{\tilde{h}_t}^{\pi^t}(\tilde{s}_t)>0$, then $p_\ell^{\pi^{t}}(s'| \tilde{s}_t,\tilde{h}_t) > 0$ implies that $p_\ell^{\pi^{t}}(s') > 0$.
Using that, under $G_{s,a,h}^t$, $(s,a,h)$  maximizes the bonuses at time $t-1$ over all triplets visited by $\pi^t$, we can upper bound the left-hand side as
\begin{align*}
    \sum_{t=1}^T \indi{G_{s,a,h}^t} \sum_{\ell=\tilde{h}_t}^H\sum_{s'\in\cS} p_\ell^{\pi^{t}}(s'| \tilde{s}_t,\tilde{h}_t) b_\ell^{t-1}(s',\pi_\ell^{t}(s')) 
    &\leq H\sum_{t=1}^T \indi{G_{s,a,h}^t} b_{h}^{t-1}(s,a) 
    \\ &\stackrel{(a)}{\leq} 2H^2 \sum_{t=1}^T \indi{G_{s,a,h}^t} \sqrt{\frac{{\beta}(\overline{n}_{h}^{t-1}(s,a),\delta)}{\overline{n}_{h}^{t-1}(s,a) \vee 1}}
    \\ &\stackrel{(b)}{\leq} 2H^2\sum_{t=1}^T \indi{G_{s,a,h}^t}\sqrt{\frac{\beta(T,\delta)}{Z_{h}^{t-1}(s,a) p_{h}^{\min}(s,a) \vee 1}}
    \\ &\stackrel{(c)}{\leq} 4H^2\sqrt{\frac{\beta(T,\delta) Z_{h}^T(s,a)}{p_{h}^{\min}(s,a)}}.
\end{align*}
where (a) uses Lemma 7 of \cite{Kaufmann21RFE} together with the definition of $b_{h}^{t-1}(s,a)$, (b) uses that $\overline{n}_{h}^{t-1}(s,a) \geq \sum_{j=1}^{t-1} \indi{G_{s,a,h}^j} p_{h}^{\pi^j}(s,a) \geq Z_{h}^{t-1}(s,a) p_{h}^{\min}(s,a)$, and (c) uses the pigeon-hole principle (see Lemma \ref{lem:pigeon-hole}). Plugging this into \eqref{eq:ineq-target-trick} and solving the resulting inequality in $Z_{h}^T(s,a)$, we obtain
\begin{align*}
    Z_{h}^T(s,a) \leq \frac{64H^4\beta(T,\delta)}{p_{h}^{\min}(s,a)\widetilde{\Delta}_{h}(s,a)^2}.
\end{align*}
A similar derivation using the stopping rule definition together with the fact that the algorithm did not stop at $T$ also allows us to prove that $Z_{h}^T(s,a) \leq \frac{144H^4\beta(T,\delta)}{p_{h}^{\min}(s,a)\epsilon^2}$ (see Lemma \ref{lem:bound-epsilon}). Plugging these two bounds into \eqref{eq:tau-target-trick} with $T = \tau-1$,
\begin{align}\label{eq:tau-implicit-bound}
    \tau \leq \sum_{h=1}^H \sum_{s\in\cS}\sum_{a\in\cA} \frac{144H^4\beta(\tau-1,\delta)}{p_{h}^{\min}(s,a) \max \{\widetilde{\Delta}_{h}(s,a),\epsilon \}^2} + 1.
\end{align}
The proof is concluded by noting that $\beta(\tau-1,\delta) \leq 5\log\frac{3SAH}{\delta} + 4S + 4S\log\left(\tau\right)$ (see Lemma \ref{lem:bound-beta}) and by using Lemma \ref{lem:simplify-ineq} to solve the resulting inequality in $\tau$ (see Appendix \ref{app:conclude-proof}). \hfill $\blacksquare$


\section{On the Regret-to-PAC Conversion}

In the minimax setting, the complexity of PAC RL and that of regret minimization are very related. Indeed, \cite{Jin18OptQL} suggest the following regret-to-PAC conversion: one can take a regret minimizer, run it for $T$ episodes, and output a policy $\widehat{\pi}$ uniformly drawn from the $T$ played. Then, by Markov's inequality, $\mathbb{P}\left(V_1^{\widehat{\pi}}(s_1) < V_1^\star(s_1) - \epsilon\right) \leq \tfrac{1}{T\epsilon}\sum_{t=1}^T \bE[ V_1^\star(s_1) - V_1^{\pi^t}(s_1)] = \frac{1}{\epsilon}\bE[\mathcal{R}_{\cM}(T)/T]$. Thus, choosing $T$ such that the expected average regret is smaller than $\epsilon\delta$ yields an $\epsilon$-optimal policy with probability $1-\delta$. This is why in the literature it is common to derive an upper bound $\overline{R}(T)$ on the expected average regret and then claim that the resulting sample complexity for PAC RL is $T_\epsilon := \inf_{T\in\mathbb{N}}\left\{T : \overline{R}(T) \leq \epsilon\delta \right\}$. However, this claim can be misleading. 

Applying this regret-to-PAC conversion to the UCBVI algorithm with Bernstein bonuses \citep{Azar17UCBVI}, we get a sample complexity of order $O(SAH^3\log(1/\delta)/(\epsilon^2 \delta^2))$, which is optimal in a minimax sense in all dependencies except $\delta$.\footnote{The dependence on $\delta$ can be improved to $\log(1/\delta)^2$, see Appendix F of \cite{Kaufmann21RFE}.} However, this trick can only be perfomed when $\overline{R}(T)$ contains quantities known by the algorithm (e.g., it can be a worst-case bound but not an instance-dependent one). In fact, the regret minimizer is used as a \emph{sampling rule} for PAC identification coupled with a \emph{deterministic stopping rule} which simply stops after $T_\epsilon$ episodes. When $T_\epsilon$ is unknown, we need to use an \emph{adaptive} stopping rule, in which case the claimed sample complexity $T_\epsilon$ might not be attainable. This is proved in the following theorem, where we show that there exist MDPs where $T_\epsilon$ can be exponentially (in $S,A$) smaller than the actual stopping time of any $(\epsilon,\delta)$-PAC algorithm.

\begin{theorem}\label{th:regret-vs-bpi}
For any $S \geq 4$, $A\geq 2$ and $H \geq \lceil\log_2(S)\rceil + 1$, there exists an MDP $\cM$ with $S$ states, $A$ actions, and horizon $H$, and a regret minimization algorithm such that
\begin{align*}
T_\epsilon := \inf_{T\in\mathbb{N}}\left\{T : \frac{1}{T}\sum_{t=1}^T \bE_{\cM}\left[ V_1^\star(s_1) - V_1^{\pi^t}(s_1) \right] \leq \epsilon\delta \right\} \leq
\frac{2}{\epsilon^2\delta}\left( 36\log(2SAH) + 16\log \frac{17}{\epsilon^2\delta} + 9\epsilon^2 \right) + 1.
\end{align*}
Moreover, on the same instance any $(\epsilon,\delta)$-PAC identification algorithm must satisfy
\begin{align*}
\bE_{\cM}[\tau] \geq \frac{SA\log(1/4\delta)}{16\epsilon^2}.
\end{align*}
\end{theorem}

Our proof (see Appendix \ref{app:regret-vs-pac}) essentially builds an MDP instance with many optimal actions. The intuition is that, in such MDP, it is relatively easy for a regret minimizer to start behaving near optimally (i.e., to have average regret below $\epsilon\delta$). However, when this occurs the regret minimizer has still not enough confidence to produce an $\epsilon$-optimal policy with probability at least $1-\delta$. That is, a stopping rule for identification would not trigger, hence the separation between the two times. 

The main implication is that the time $T_\epsilon$ at which the average regret goes below $\epsilon\delta$ is not always a good proxy for the sample complexity that a regret minimizer would take for $(\epsilon,\delta)$-PAC identification. In particular, one cannot simply take an existing instance-dependent regret bound \citep[e.g.,][]{simchowitz2019non,dann21ReturnGap,xu2021fine} and turn it into a sample complexity bound by the regret-to-PAC conversion suggested above. A specific analysis for the PAC setting, like the one we propose in Section \ref{sec:analysis} or those of \cite{wagenmaker21IDPAC,tirinzoni22deterministic}, is actually needed.

Finally, we note that this result also solves an open question left by \cite{wagenmaker21IDPAC} in their conclusion. First, it shows that the sample complexity stated in Equation (7.1) of \cite{wagenmaker21IDPAC} for a regret-to-PAC conversion from an instance-dependent regret bound cannot always be attained by a PAC RL algorithm. Second, it shows that the extra term $|\mathrm{OPT}(\epsilon)|/\epsilon^2$ that appears in the complexity of MOCA is actually tight, at least in a worst-case sense, as our proof essentially builds an MDP where all $\epsilon$-optimal state-action pairs must be visited $\Omega(1/\epsilon^2)$ times.


\section{Discussion}

We derived the first instance-dependent sample complexity bound for an optimistic sampling rule (BPI-UCRL). It features a new notion of sub-optimality gap that we call ``conditional return gap'' and that is tighter than existing value gaps and (deterministic) return gaps. We proved this bound with a remarkably simple analysis based on a new ``target trick'' that could be of independent interest. We complemented this result by showing that one cannot directly leverage the standard regret-to-PAC conversion in the instance-dependent regime, thus making our novel analysis non-trivial.

In the bandit setting, it is known that optimism, when coupled with an appropriate stopping and recommendation rule, is near instance-optimal for best-arm identification with (sub)Gaussian distributions \citep{jamieson14LILUCB}. In this work, we obtained a similar result for deterministic MDPs, where optimistic sampling rules are sub-optimal only by a factor $H^3$. This also explains the good empirical performance of BPI-UCRL observed by \cite{tirinzoni22deterministic} in such a setting. However, there seems to be a large gap for general stochastic MDPs, where our sample complexity scales with some minimal visitation probabilities that are avoided by algorithms like MOCA. This can be related to known results for structured bandits \citep{lattimore2017end}, as a stochastic MDP presents a complex trade-off between collecting rewards and gathering information (i.e., exploring the state space) for which an optimistic algorithm can be arbitrarily sub-optimal.

Finding the right complexity (matching upper and lower bounds) for PAC RL in general stochastic MDPs remains the main open problem. In deterministic MDPs, upper and lower bounds are nearly matching and are expressed as (complex) functions of the (simple) deterministic return gaps \citep{tirinzoni22deterministic}. They were obtained by properly combining a coverage-based exploration strategy with a suitable elimination rule. We conjecture that a similar algorithmic design could be a good direction towards instance optimality in stochastic MDPs. This would involve the combination of (1) a coverage-based exploration strategy like MOCA \citep{wagenmaker21IDPAC} that ensures scaling with the ``right'' visitation probabilities, and (2) some elimination rule to avoid over-sampling that ensures scaling with the ``right'' notion of gap.
Unfortunately, an instance-dependent lower bound for the general setting is still unknown and thus it remains unclear what these ``right'' notions are. In this work, we take a step forward by proposing a novel and tighter gap definition, though it remains an open question whether our conditional return gaps can be related to an actual sample complexity lower bound.




\vskip 0.2in
\bibliography{biblio_bpi}

\newpage

\appendix 


\section{Proofs of Section \ref{sec:analysis}}\label{app:analysis}

\subsection{Additional notation}

We define the following upper and lower confidence bounds over the value functions of each policy $\pi$. We initialize $\underline{V}_{H+1}^{t,\pi}(s) = \overline{V}_{H+1}^{t,\pi}(s) = 0$, then we define recursively
\begin{align*}
  \overline{Q}_h^{t,\pi}(s,a) &= \min\bigg(H-h+1,\ \widehat{r}_h^t(s,a) + b_h^t(s,a) + \sum_{s' \in \cS}\widehat{p}_h^t(s' |s,a) \overline{V}_{h+1}^{t,\pi}(s')\bigg), \quad
  &&\overline{V}_h^{t,\pi}(s) = \overline{Q}_h^{t,\pi}(s,\pi_h(s)),\\
  \underline{Q}_h^{t,\pi}(s,a) &= \max\bigg(0,\ \widehat{r}_h^t(s,a) - b_h^t(s,a) + \sum_{s' \in \cS}\widehat{p}_h^t(s' |s,a) \underline{V}_{h+1}^{t,\pi}(s')\bigg), \quad
  &&\underline{V}_h^{t,\pi}(s) = \underline{Q}_h^{t,\pi}(s,\pi_h(s)).
\end{align*}

\subsection{Proof of Lemma \ref{lem:policy-value-gap-vs-conf}}

Using event $\cE$ and the fact that $\pi^{t+1}$ is greedy w.r.t. $\overline{Q}_h^{t}(s,a)$,
    \begin{align*}
        V_h^\star(s) - V_h^{\pi^{t+1}}(s) = \max_{a\in\cA} Q_h^\star(s,a) - Q_h^{\pi^{t+1}}(s,\pi^{t+1}_h(s)) &\leq \max_{a\in\cA} \overline{Q}_h^{t}(s,a) - Q_h^{\pi^{t+1}}(s,\pi^{t+1}_h(s))
        \\ &= \overline{Q}_h^{t,\pi^{t+1}}(s,\pi_h^{t+1}(s)) - Q_h^{\pi^{t+1}}(s,\pi^{t+1}_h(s)).
    \end{align*}
Let $a = \pi^{t+1}_h(s)$. Expanding the last quantity using the Bellman equations,
    \begin{align*}
        \overline{Q}_h^{t,\pi^{t+1}}(s,a) - Q_h^{\pi^{t+1}}(s,a) 
        &\leq \widehat{r}_h^t(s,a) - r_h(s,a) + \sum_{s'\in\cS} (\widehat{p}_h^t(s'|s,a)-p_h(s'|s,a))\overline{V}_{h+1}^{t,\pi^{t+1}}(s')
        \\ & \qquad\qquad\qquad\qquad\quad + \sum_{s'\in\cS} p_h(s'|s,a) \Big(\overline{V}_{h+1}^{t,\pi^{t+1}}(s') - V_{h+1}^{\pi^{t+1}}(s')\Big) + b_h^t(s,a)
        \\ &\leq \sqrt{\frac{\beta^{r}(n_h^t(s,a),\delta)}{n_h^{t}(s,a)\vee 1}} \wedge 1 + (H-h)\sqrt{\frac{2\beta^{p}(n_h^t(s,a),\delta)}{n_h^{t}(s,a)\vee 1}} \wedge (H-h)
        \\ & \qquad\qquad\qquad\qquad\quad + \sum_{s'\in\cS} p_h(s'|s,a) \Big(\overline{V}_{h+1}^{t,\pi^{t+1}}(s') - V_{h+1}^{\pi^{t+1}}(s')\Big) + b_h^t(s,a)
        \\ &\leq 2b_h^t(s,a) + \sum_{s'\in\cS} p_h(s'|s,a) \Big(\overline{V}_{h+1}^{t,\pi^{t+1}}(s') - V_{h+1}^{\pi^{t+1}}(s')\Big),
    \end{align*}
where in the second inequality we used event $\cE$ as in the proof of Lemma \ref{lem:conf-Q}. The statement follows by recursively applying this reasoning to $\overline{V}_{h+1}^{t,\pi^{t+1}}(s') - V_{h+1}^{\pi^{t+1}}(s') = \overline{Q}_{h+1}^{t,\pi^{t+1}}(s',\pi_{h+1}^{t+1}(s')) - Q_{h+1}^{\pi^{t+1}}(s',\pi_{h+1}^{t+1}(s'))$. \hfill $\blacksquare$

\subsection{Other results}

\begin{lemma}\label{lem:good-event} Using the threshold $\beta$ defined in Theorem \ref{th:main-sample-complexity}, $\bP(\cE)\geq 1 - \delta$.
\end{lemma}
\begin{proof}
    $\cE^r$ and $\cE^p$ hold with probability at least $1-\delta/3$ each by applying Proposition 1 and 2 of \cite{MDPGapE} together with a union bound and Pinsker's inequality for the rewards. $\cE^c$ holds with probability at least $1-3\delta$ by Lemma F.4 of \cite{dann2017unifying} and a union bound. Another union bound over the three events proves the statement.
\end{proof}

\begin{lemma}\label{lem:conf-Q}
    Using the threshold $\beta$ defined in Theorem \ref{th:main-sample-complexity}, under event $\cE$, for any $t\in\mathbb{N}_{>0}$, $s\in\cS$, $a\in\cA$, $h\in[H]$,
    \begin{align*}
        \underline{Q}_h^{t,\pi}(s,a) &\leq {Q}_h^{\pi}(s,a) \leq \overline{Q}_h^{t,\pi}(s,a),
        \\ \underline{Q}_h^{t}(s,a) &\leq {Q}_h^{\star}(s,a) \leq \overline{Q}_h^{t}(s,a).
    \end{align*}
\end{lemma}
\begin{proof}
    Clearly, all inequalities hold at stage $H$ since ${Q}_H^{\pi}(s,a) = {Q}_H^{\star}(s,a) = r_H(s,a)$ and, by event $\cE^r$, together with the fact that rewards are bounded in $[0,1]$,
    \begin{align*}
        \Big| r_H(s,a) - \widehat{r}^{t}_H(s,a) \Big| \leq \sqrt{\frac{\beta^{r}(n_h^t(s,a),\delta)}{n_h^{t}(s,a)\vee 1}} \wedge 1 \leq b_H^t(s,a).
    \end{align*}
    Now suppose the inequalities hold at stage $h+1 \leq H$. At stage $h$, we have
    \begin{align*}
        Q_h^\pi(s,a) &= r_h(s,a) + \sum_{s'\in\cS} p_h(s'|s,a)V_{h+1}^\pi(s')
        \\ &\stackrel{(a)}{\leq} r_h(s,a) + \sum_{s'\in\cS} p_h(s'|s,a)\overline{V}_{h+1}^{t,\pi}(s')
        \\ &\stackrel{(b)}{\leq} r_h(s,a) + \sum_{s'\in\cS} \widehat{p}_h^t(s'|s,a)\overline{V}_{h+1}^{t,\pi}(s') + (H-h)\| p_h(s,a) - \widehat{p}_h^t(s,a) \|_1
        \\ &\stackrel{(c)}{\leq} r_h(s,a) + \sum_{s'\in\cS} \widehat{p}_h^t(s'|s,a)\overline{V}_{h+1}^{t,\pi}(s') + (H-h)\sqrt{2\mathrm{KL}(\widehat{p}_h^t(s,a),p_h(s,a))}
        \\ &\stackrel{(d)}{\leq} \widehat{r}_h^t(s,a) + \sqrt{\frac{\beta^{r}(n_h^t(s,a),\delta)}{n_h^{t}(s,a)\vee 1}} + \sum_{s'\in\cS} \widehat{p}_h^t(s'|s,a)\overline{V}_{h+1}^{t,\pi}(s') + (H-h)\sqrt{\frac{2\beta^{p}(n_h^t(s,a),\delta)}{n_h^{t}(s,a)\vee 1}},
    \end{align*}
    where (a) is by assumption, (b) uses that $\overline{V}_{h+1}^{t,\pi}(s')$ is bounded by $H-h$, (c) is from Pinsker's inequality, and (d) uses the event $\cE$. As before, since rewards are bounded in $[0,1]$ and $\sum_{s'\in\cS} (p_h(s'|s,a) - \widehat{p}_h^t(s'|s,a))\overline{V}_{h+1}^{t,\pi}(s') \leq H-h$, we can clip the two bonuses above to $1$ and $H-h$, respectively. This implies that,
    \begin{align*}
        \sqrt{\frac{\beta^{r}(n_h^t(s,a),\delta)}{n_h^{t}(s,a)\vee 1}} \wedge 1 &+ (H-h)\sqrt{\frac{2\beta^{p}(n_h^t(s,a),\delta)}{n_h^{t}(s,a)\vee 1}} \wedge (H-h)
        \\ & \leq \sqrt{\frac{\beta(n_h^t(s,a),\delta)}{n_h^{t}(s,a)\vee 1}} \wedge 1 + (H-h)\sqrt{\frac{\beta(n_h^t(s,a),\delta)}{n_h^{t}(s,a)\vee 1}} \wedge (H-h)
        \\ &\leq (H-h+1)\sqrt{\frac{\beta(n_h^t(s,a),\delta)}{n_h^{t}(s,a)\vee 1}} \wedge (H-h+1) \leq b_h^t(s,a).
    \end{align*}
    This proves that $Q_h^\pi(s,a) \leq \overline{Q}_h^{t,\pi}(s,a)$. The proofs of all other inequalities follow analogously.
\end{proof}


\begin{lemma}\label{lem:upper-bound-stopping} Under event $\cE$, for all $(s,a)$ and $h \leq H$, 
 \begin{align*}
        \max_{a} \overline Q_1^{t}(s_1,a) - \max_{a} \underline Q_1^{t}(s_1,a) \leq 3\sum_{h=1}^H\sum_{s,a} p_h^{{\pi}^{t+1}}(s,a)b_h^t(s,a),
    \end{align*}
\end{lemma}

\begin{proof} By definition of the optimistic rule, we first observe that 
 \[    \max_{a} \overline Q_1^{t}(s_1,a) - \max_{a} \underline Q_1^{t}(s_1,a) =     \overline Q_1^{t}(s_1,\pi_h^{t+1}(s_1) - \max_{a} \underline Q_1^{t}(s_1,a) \leq D_1^{t}(s_1,\pi_1^{t+1}(s_1)) \]
where we introduce the diameters
\[D_h^{t}(s,a) :=  \overline Q_h^{t}(s,a) - \underline Q_h^{t}(s,a).\]
Using the inductive definition of the confidence bounds, we get 
\begin{eqnarray*}D_h(s,a) & \leq & 2b_h^{t}(s,a) + \sum_{s' \in \cS} \widehat{p}_h^{t}(s' | s,a) \left(\max_{a}\overline Q_{h+1}^{t}(s,a) - \max_{a}\underline Q_{h+1}^{t}(s,a)\right) \\ 
& \leq & 2b_h^{t}(s,a) + \sum_{s' \in \cS} \widehat{p}_h^{t}(s' | s,a) D_{h+1}^t(s', \pi_{h+1}^{t+1}(s')) \\
& = & 2b_h^{t}(s,a) + \sum_{s' \in \cS}( \widehat{p}_h^{t}(s' | s,a) - p_h(s' |s,a)) D_{h+1}^t(s', \pi_{h+1}^{t+1}(s')) +  \sum_{s' \in \cS}p_h(s' |s,a)) D_{h+1}^t(s', \pi_{h+1}^{t+1}(s')) \\
& \leq & 3b_h^{t}(s,a) + \sum_{s' \in \cS}p_h(s' |s,a) D_{h+1}^t(s', \pi_{h+1}^{t+1}(s')),\end{eqnarray*}
and the result follows by induction.
\end{proof}

\begin{lemma}\label{lem:pigeon-hole}
    For any $T > 0$, $s\in\cS$, $a\in\cA$, $h\in[H]$,
    \begin{align*}
        \sum_{t=1}^T \indi{G_{s,a,h}^t} \sqrt{\frac{1}{Z_{h}^{t-1}(s,a) p_{h}^{\min}(s,a) \vee 1}} \leq 2\sqrt{\frac{Z_{h}^T(s,a)}{p_{h}^{\min}(s,a)}}.
    \end{align*}
\end{lemma}
\begin{proof}
    Using the pigeon-hole principle together with the inequality $\sum_{i=1}^n 1/\sqrt{i} \leq 2\sqrt{n}-1$,
    \begin{align*}
        \sum_{t=1}^T \indi{G_{s,a,h}^t} \sqrt{\frac{1}{Z_{h}^{t-1}(s,a) p_{h}^{\min}(s,a) \vee 1}} 
        & \leq 1 + \frac{1}{\sqrt{p_{h}^{\min}(s,a)}}\sum_{j=2}^{Z_h^T(s,a)} \sqrt{\frac{1}{j-1}} 
        \\ &\leq 1 + \frac{2\sqrt{Z_{h}^T(s,a)} - 1}{\sqrt{p_{h}^{\min}(s,a)}}
        \leq 2\sqrt{\frac{Z_{h}^T(s,a)}{p_{h}^{\min}(s,a)}}.
    \end{align*}
\end{proof}

\begin{lemma}\label{lem:bound-epsilon}
    For any time $T > 0$ at the end of which the algorithm did not stop, for any $s\in\cS,a\in\cA,h\in[H]$,
    \begin{align*}
        Z_{h}^T(s,a) \leq \frac{144H^4\beta(T,\delta)}{p_{h}^{\min}(s,a)\epsilon^2}.
    \end{align*}
\end{lemma}
\begin{proof}
    If the algorithm did not stop at the end of time $T$, by the definition of the stopping rule and Lemma \ref{lem:upper-bound-stopping}, for all $t\leq T$,
\begin{align*}
    \epsilon \leq \max_{a} \overline Q_1^{t}(s_1,a) - \max_{a} \underline Q_1^{t}(s_1,a) \leq 3\sum_{h=1}^H\sum_{s,a} p_h^{{\pi}^{t+1}}(s,a)b_h^t(s,a).
\end{align*}
Summing both sides over times where $(s,a,h)$ is targeted,
\begin{align*}
    \epsilon Z_{h}^T(s,a) 
    = \epsilon \sum_{t=1}^T \indi{G_{s,a,h}^t}
    &\leq 3\sum_{t=1}^T \indi{G_{s,a,h}^t}\sum_{\ell=1}^H\sum_{s',a'} p_\ell^{{\pi}^{t}}(s',a')b_\ell^{t-1}(s',a')
    \\ &\leq 3H \sum_{t=1}^T \indi{G_{s,a,h}^t} b_{h}^{t-1}(s,a) \leq 12H^2 \sqrt{\frac{Z_{h}^T(s,a)\beta(T,\delta)}{p_{h}^{\min}(s,a)}},
\end{align*}
where the last inequality was already derived in the proof of Theorem \ref{th:main-sample-complexity}. The statement follows by solving the resulting inequality in $Z_h^T(s,a)$.
\end{proof}

\begin{lemma}\label{lem:bound-beta}
    Let $S\geq 2$. For any time $t\geq 1$, $\beta(t-1,\delta) \leq 5\log\frac{3SAH}{\delta} + 4S + 4S\log\left(t\right)$.
\end{lemma}
\begin{proof}
    Starting from the definition of $\beta$ and using the inequality $(x+y)^2 \leq 2x^2 + 2y^2$,
    \begin{align*}
        \beta(t-1,\delta) &= \left(\sqrt{\frac{1}{2}\left(\log\frac{3SAH}{\delta} + \log(et)\right)} + \sqrt{2\log\frac{3SAH}{\delta} + 2(S-1)\log\left(e\left(1+\frac{t-1}{S-1}\right)\right)}\right)^2
        \\ &\leq 5\log\frac{3SAH}{\delta} + \log(et) + 4(S-1)\log\left(e\left(1+\frac{t-1}{S-1}\right)\right)
        \\ &\leq 5\log\frac{3SAH}{\delta} + \log(et) + 4(S-1)\log\left(et\right)
        \\ &\leq 5\log\frac{3SAH}{\delta} + 4S\log\left(et\right)
        \\ &= 5\log\frac{3SAH}{\delta} + 4S + 4S\log\left(t\right).
    \end{align*}
\end{proof}

\begin{lemma}\label{lem:simplify-ineq}
    Let $B,C \geq 1$. If $k \leq B\log(k) + C$, then
    \begin{align*}
    k \leq B\log(B^2 + 2C) + C.
    \end{align*}
\end{lemma}
\begin{proof}
    Since $\log(k) \leq \sqrt{k}$ for any $k\geq 1$, we have that $k \leq B\sqrt{k} + C$.
    Solving this second-order inequality, we get the crude bound $\sqrt{k} \leq \frac{B}{2} + \sqrt{\frac{B^2}{4} + C}$, which in turns yields $k \leq B^2 + 2C$ using that $(x+y)^2 \leq 2(x^2+y^2)$ for $x,y\geq 0$. The statement follows by plugging this bound into the logarithm.
\end{proof}

\subsection{Explicit sample complexity bound}\label{app:conclude-proof}

We show how to derive the sample complexity bound stated in Theorem \ref{th:main-sample-complexity} starting from the one derived in \eqref{eq:tau-implicit-bound}. Let $\cC(\epsilon) := \sum_{h=1}^H \sum_{s\in\cS}\sum_{a\in\cA} \frac{H^4}{p_{h}^{\min}(s,a) \max \{\widetilde{\Delta}_{h}(s,a),\epsilon \}^2}$. Since $\beta(\tau-1,\delta) \leq 5\log\frac{3SAH}{\delta} + 4S + 4S\log\left(\tau\right)$ from Lemma \ref{lem:bound-beta}, \eqref{eq:tau-implicit-bound} implies that
\begin{align*}
    \tau &\leq 720 \cC(\epsilon)\log\frac{3SAH}{\delta} + 576 \cC(\epsilon)S + 576\cC(\epsilon)S\log\left(\tau\right) + 1
    \\ &\leq 720 \cC(\epsilon)\log\frac{3SAH}{\delta} + 577 \cC(\epsilon)S + 576\cC(\epsilon)S\log\left(\tau\right),
\end{align*}
where the second inequality holds since $1 \leq \cC(\epsilon)S$. Using Lemma \ref{lem:simplify-ineq} with $B = 576\cC(\epsilon)S$ and $C = 720 \cC(\epsilon)\log\frac{3SAH}{\delta} + 577 \cC(\epsilon)S$,
\begin{align*}
    \tau &\leq 576\cC(\epsilon)S\log\left(576^2\cC(\epsilon)^2S^2 + 1440 \cC(\epsilon)\log\frac{3SAH}{\delta} + 1154 \cC(\epsilon)S\right) + 720 \cC(\epsilon)\log\frac{3SAH}{\delta} + 577 \cC(\epsilon)S
    \\ &\leq 576\cC(\epsilon)S\log\left(4 \cdot 576^2\cC(\epsilon)^2S^2 \log\frac{3SAH}{\delta} \right) + 720 \cC(\epsilon)\log\frac{3SAH}{\delta} + 577 \cC(\epsilon)S
    \\ &\leq 1152\cC(\epsilon)S\log\left( 1152\cC(\epsilon)S \log\frac{3SAH}{\delta} \right) + 720 \cC(\epsilon)\log\frac{3SAH}{\delta} + 577 \cC(\epsilon)S
    \\ &\leq 1729\cC(\epsilon)S\log\left( 1152\cC(\epsilon)S \log\frac{3SAH}{\delta} \right) + 720 \cC(\epsilon)\log\frac{3SAH}{\delta},
\end{align*}
where the inequalities use some trivial bounds to simplify the final expression. The result stated in Theorem \ref{th:main-sample-complexity} follows from here by noting that $\cC(\epsilon) \leq \frac{SAH^5}{p_{\min}\epsilon^2}$.


\section{Proof of Theorem \ref{th:regret-vs-bpi}} \label{app:regret-vs-pac}

\begin{figure}
\centering
\includegraphics[scale=0.8]{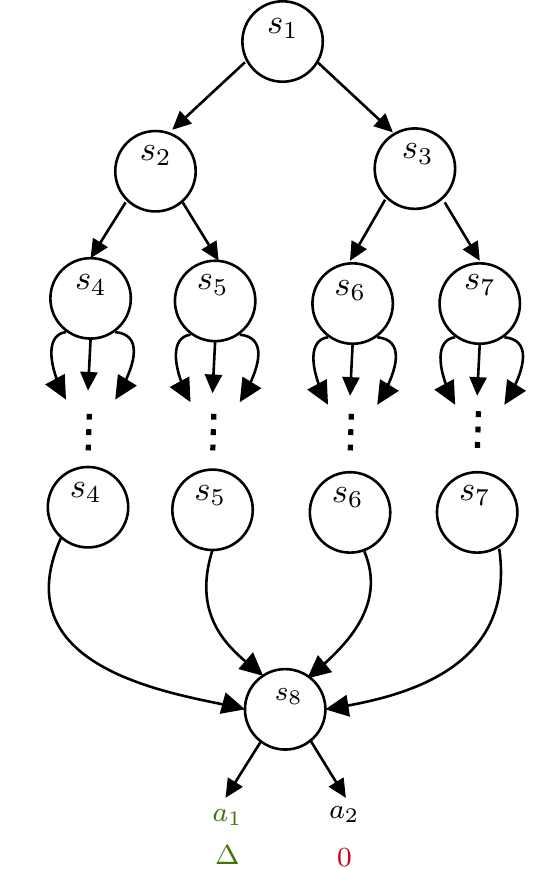}
\caption{Example of the MDP instance for proving Theorem \ref{th:regret-vs-bpi} where $S=8$, $A=3$, and $H \geq 4$. The reward is Gaussian with variance $1$ and zero mean except for $r_H(s_8,a_1)= \Delta > 0$. The optimal policy takes action $a_1$ in $s_8$ at the last stage $H$, while any path of length $H-1$ to reach that state is optimal.}\label{fig:bpi-vs-regret}
\end{figure}

An example of the MDP instance we build to prove Theorem \ref{th:regret-vs-bpi} is shown in Figure \ref{fig:bpi-vs-regret}. Suppose $S \geq 4$, $A\geq 2$ and $H \geq \lceil\log_2(S)\rceil + 1$. We arrange the states in a binary tree, starting from the root and adding them one by one from left to right and top to bottom. The leaves of the binary tree have all $A$ actions available, and so do the states at the second-last layer which have zero children. Such actions keep the agent in the same state up to layer $H-1$. In layer $H-1$, all $A$ actions for all reachable states transition to state $s_8$. In the latter, only two actions are available, among which $a_1$ is the only in the whole MDP with a positive reward of $\Delta > 0$.

The intuition is that this is an extremely easy instance of regret minimization. In fact, it is essentially a two-armed bandit where the only thing that must be learned is the optimal action at stage $H$ (i.e., $a_1$), while the agent can behave arbitrarily in all stages before and still suffer zero regret. On the other hand, this is an extremely hard instance for PAC identification since, in order to return an $\epsilon$-optimal policy with enough confidence, any algorithm must explore all state-action pairs at stages from $1$ to $H-1$ up to an error below $\epsilon$ in order to assess that their rewards are all $\epsilon$-close.

\paragraph{Proof of the sample-complexity lower bound}

Let us start by proving the lower bound on the sample complexity of any $(\epsilon,\delta)$-PAC algorithm. Let $d$ be the depth of the tree, i.e., the first integer such that $\sum_{i=0}^{d-1} 2^i = 2^d - 1 \geq S-1$. That is $d = \lceil \log_2(S) \rceil$. Note that, even if the last layer is not complete (i.e., it has less than $2^{d-1}$ states), the second last layer must be complete. Therefore, the are at least $2^{d-2} \geq S/4$ states with all $A$ actions available that are reachable at stage $H-1$. Call these states $\bar{s}_1,\dots,\bar{s}_m$ for $m$ some integer with $m\geq S/4$. Moreover, $\overline{\Delta}_{H-1}(\bar{s}_i,a) = 0$ for all $i\in[m]$ and $a\in[A]$ since all paths are optimal up to stage $H-1$, where $\overline{\Delta}$ denotes the deterministic return gaps of \cite{tirinzoni22deterministic}. Note that the MDP is deterministic, so the lower bound of Theorem 2 by \cite{tirinzoni22deterministic} holds. By applying this result, we get
\begin{align*}
\forall i\in[m], a\in[A] : \bE[n_{H-1}^\tau(\bar{s}_i,a)] \geq \frac{\log(1/4\delta)}{4\epsilon^2}.
\end{align*}
This directly implies that
\begin{align*}
\bE[\tau] = \sum_{s\in\cS_{H-1}}\sum_{a\in[A]} \bE[n_{H-1}^\tau(s,a)] \geq \frac{SA\log(1/4\delta)}{16\epsilon^2}.
\end{align*}

\paragraph{Proof of the regret upper bound}

Let us now deal with regret minimization. Let us take the UCBVI algorithm \citep{Azar17UCBVI} with Hoeffding bonus (for general stochastic transitions) that we described in Section \ref{sec:bpi-ucrl}. 
We shall consider a slightly different \emph{stage-dependent} definition of the bonuses $b_h^t(s,a)$. All we need is that, at any time $t\in\mathbb{N}$, they guarantee concentration for all $(s,a,h)$ with probability at least $1-\frac{1}{t^2}$. For our proof we only need to specify the specific form at stage $H$. Since at that stage we only need to concentrate rewards, using Hoeffding's inequality for sub-Gaussian distributions with $\sigma^2=1$, it is easy to see that
\begin{align*}
    b_H^t(s,a) := \sqrt{\frac{2\log (2SAHt^2)}{n_H^{t}(s,a)}} \wedge 1
\end{align*}
ensures $\overline{Q}_H^t(s,a) \geq Q_H^\star(s,a)$ for all $s\in \cS,a\in \cA,t\in\mathbb{N}$ with probability at least $1-\frac{1}{t^2}$. For all stages $h=1,\dots,H-1$, we can simply take the bonuses considered in the main paper with a decreasing schedule for $\delta$, though their explicit expression is not really used in our proof.

Note that, in this particular MDP, the regret is zero whenever the agent plays action $a_1$ at stage $H$ since all actions played from stage $1$ to $H-1$ are optimal. In other words, this is equivalent to a bandit problem with two actions. Therefore, for any $T\geq 1$,
\begin{align*}
\sum_{t=1}^T \Big( V_1^\star(s_1) - V_1^{\pi^t}(s_1) \Big) = \Delta \sum_{t=1}^T \indi{a_H^t = a_2} = \Delta n_H^T(s_H,a_2).
\end{align*}
Under the good event $\cG_t$ in which the confidence intervals are valid at $t$, if $a_H^t = a_2$, then
\begin{align*}
b_H^{t-1}(s_H,a_2) \geq \frac{\Delta}{2} \implies  n_H^{t-1}(s_H,a_2) \leq \frac{8}{\Delta^2}\log (2SAHt^2).
\end{align*}
Therefore, the cumulative regret up to any time $T$ in such good events can be bounded as
\begin{align*}
\sum_{t\leq T: \cG_t} \Big( V_1^\star(s_1) - V_1^{\pi^t}(s_1) \Big) \leq \Delta \sum_{t\leq T: \cG_t} \indi{a_H^t = a_2, n_H^{t-1}(s_H,a_2) \leq \frac{8}{\Delta^2}\log (2SAHT^2)} \leq \frac{8}{\Delta}\log (2SAHT^2).
\end{align*}
On the other hand, the expected regret under the bad events is bounded as
\begin{align*}
\bE\left[\sum_{t\leq T: \neg\cG_t} \Big( V_1^\star(s_1) - V_1^{\pi^t}(s_1) \Big)\right] \leq \Delta \sum_{t=1}^T \bP(\neg\cG_t) \leq \Delta \sum_{t=1}^T \frac{1}{t^2} \leq 2\Delta.
\end{align*}
Combining these two we obtain the following bound on the expected cumulative regret:
\begin{align*}
\bE\left[\sum_{t=1}^T \Big( V_1^\star(s_1) - V_1^{\pi^t}(s_1) \Big)\right] \leq \frac{8}{\Delta}\log (2SAHT^2) + 2\Delta.
\end{align*}
Finally,
\begin{align*}
T_\epsilon := \inf_{T\in\mathbb{N}}\left\{T : \frac{1}{T}\sum_{t=1}^T \bE\left[ V_1^\star(s_1) - V_1^{\pi^t}(s_1) \right] \leq \epsilon\delta \right\} \leq \inf_{T\in\mathbb{N}}\left\{T : \frac{8}{\Delta}\log (2SAHT^2) + 2\Delta \leq T\epsilon\delta \right\}.
\end{align*}
To bound $T_\epsilon$, we need to solve the inequality on the right-hand side above. Using $\log(T)\leq\sqrt{T}$, it is easy to show that a crude bound is
\begin{align*}
T \leq \frac{260}{\Delta^2\epsilon^2\delta^2}\left( 4\log(2SAH) + \Delta^2\right).
\end{align*}
Plugging this into the logarithm above yields
\begin{align*}
T_\epsilon \leq \frac{2}{\Delta\epsilon\delta}\left( 4\log(2SAH) + 16\log \frac{17}{\Delta\epsilon\delta} + 8\log\big(4\log(2SAH) + \Delta^2\big) + \Delta^2 \right) + 1.
\end{align*}
Setting $\Delta = \epsilon$ and using $\log\big(4\log(2SAH) + \Delta^2\big) \leq 4\log(2SAH) + \Delta^2$ concludes the proof. \hfill $\blacksquare$

\end{document}